\documentclass{article} 
\usepackage{nips14submit_e,times}
\usepackage{hyperref}
\usepackage{url}
\usepackage[numbers]{natbib}
\usepackage{helvet}
\usepackage{courier}
\usepackage{graphicx,graphics}
\usepackage{color}
\usepackage{amsmath, amssymb}
\usepackage[amsmath,thmmarks]{ntheorem}

\newtheorem{thm}{Theorem}
\newtheorem{lemma}{Lemma}

\newtheorem{dfn}{Definition}
\newtheorem{ex}{Example}

\newenvironment{proof}{\vspace{-1mm}{\bf Proof:}}{\hfill $\Box$ }
\newenvironment{sketch}{\vspace{-1mm}\noindent{\bf Proof sketch:}}{\hfill $\Box$ }
\newcount\Comments
\newcommand{\kibitz}[2]{\ifnum\Comments=1\textcolor{#1}{#2}\fi}

\newcommand{\pr}{\text{Pr}}
\newcommand{\kendall}{\text{Kendall}}
\newcommand{\kemeny}{\text{Kemeny}}

\newcommand{\mb}{{\mathcal B}}

\newcommand{\mf}{{\mathcal F}}

\newcommand{\ra}{\rightarrow}

\newcommand{\mm}{{ \mathcal M}}

\newcommand{\ms}{{\mathcal S}}
\newcommand{\ml}{{\mathcal L}}
\newcommand{\mc}{{\mathcal C}}
\newcommand{\md}{{\mathcal D}}

\newcommand{\mn}{{\mathcal N}}

\newcommand\Omit[1]{ }

\begin{document}
%
\title{A Statistical Decision-Theoretic Framework for Social Choice}
\nipsfinalcopy
\author{Hossein Azari Soufiani\thanks{azari@google.com, Google Research,  New York, NY 10011, USA. The work was done when the author was at Harvard University.} \And David C.~Parkes \thanks{parkes@eecs.harvard.edu, Harvard University, Cambridge, MA 02138, USA.}\And Lirong Xia\thanks{xial@cs.rpi.edu, Rensselaer Polytechnic Institute, Troy, NY 12180, USA.}
}
\maketitle
\begin{abstract}
%
In this paper, we take a statistical decision-theoretic viewpoint on social choice, putting a focus on
the decision to be made on behalf of a system of
agents.
In our framework, we are given a statistical ranking
model, a decision space, and a loss function defined on (parameter, decision) pairs, and formulate 
social choice mechanisms as decision rules that minimize expected loss.
This suggests a general framework for the design and analysis of new social choice mechanisms. 
We compare {\em Bayesian estimators}, which minimize Bayesian expected loss, for the Mallows model and the Condorcet model respectively, and the Kemeny rule.
We consider various normative properties, in addition to computational complexity and asymptotic behavior. In particular, we show that the Bayesian estimator for the Condorcet model satisfies some desired properties such as anonymity, neutrality, and monotonicity, can be computed in polynomial time, and is asymptotically different from the other two rules when the data are generated from the Condorcet model for some ground truth parameter.
\end{abstract}

\section{Introduction}

Social choice studies the design and evaluation
of voting rules (or rank aggregation rules).
There have been two main perspectives: reach a compromise among
subjective preferences of agents, or make an objectively correct
decision. The former has been extensively studied in classical social
choice in the context of political elections, while the latter is relatively less developed, even though it can be dated back to the Condorcet Jury
Theorem in the 18th century~\citep{Condorcet1785:Essai}.

In many multi-agent and social choice scenarios the main consideration is to achieve the
second objective, and make an objectively correct decision. Meanwhile, we also want to respect agents' preferences and opinions, and require the voting rule to satisfy well-established normative properties in social choice. For example, when a group of friends vote to choose a restaurant for dinner, perhaps the most important goal is to find an objectively good restaurant, but it is also important to use a good voting rule in the social choice sense. Even for applications with less societal context, e.g.~using voting rules to aggregate rankings in meta-search engines~\cite{Dwork01:Rank}, recommender systems~\cite{Ghosh99:Voting},  crowdsourcing~\cite{Mao13:Better}, semantic webs~\cite{Porello13:Ontology}, some social choice normative properties are still desired. For example, {\em monotonicity} may be desired, which requires that raising the position of an alternative in any vote does not hurt the alternative in the outcome of the voting rule. In addition, we require voting rules to be efficiently computable.

Such scenarios propose the following new challenge: 
{\em How can we design new voting rules with good statistical properties as well as social choice normative properties?}



To tackle this challenge, we develop a general framework
that  adopts {\em statistical decision
  theory}~\citep{Berger85:Statistical}.
Our approach 
couples a statistical
ranking model with an explicit decision space
and loss function.
 Given these, we can adopt {\em Bayesian estimators} as
social choice mechanisms, which 
make decisions to minimize the
expected loss w.r.t.~the posterior distribution on the parameters
(called the {\em Bayesian risk}). 
This provides a principled
methodology for the design and analysis of new voting rules.

To show the viability of the framework, we focus on selecting multiple alternatives (the alternatives that can be thought of as being ``tied'' for the first place) under a natural extension of the $0$-$1$ loss function 
for two models: let $\mm_\varphi^1$ denote the {\em Mallows model} with fixed dispersion~\cite{Mallows57:Non-null}, and let $\mm_\varphi^2$ denote the {\em Condorcet model} proposed by Condorcet in the 18th century~\cite{Condorcet1785:Essai,Young88:Condorcet}. In both models the dispersion parameter,
denoted $\varphi$, is taken as a fixed parameter. The difference is that in the Mallows model the parameter space is composed of all linear orders over alternatives,  while in the Condorcet model the parameter space is composed of all possibly cyclic rankings over alternatives (irreflexive, antisymmetric, and total binary relations). $\mm_\varphi^2$ is a natural model that captures real-world scenarios where the ground truth may contain cycles, or agents' preferences are cyclic, but they have to report a linear order due to the protocol. More importantly, as we will show later, a Bayesian estimator on $\mm_\varphi^2$ is superior from a computational viewpoint.

Through this approach, we obtain two voting rules as Bayesian estimators and then evaluate them with respect to various normative properties, including {\em
  anonymity, neutrality, monotonicity, the majority criterion, the Condorcet
  criterion} and {\em consistency}. Both rules satisfy anonymity, neutrality, and monotonicity, but fail the majority criterion, Condorcet criterion,\footnote{The new voting rule for $\mm_\varphi^1$ fails them for all $\varphi<1/\sqrt 2$.} and consistency. Admittedly, the two rules do not enjoy outstanding normative properties, but they are not bad either. 
We also investigate the
computational complexity of the two  rules.  Strikingly, despite
the similarity of the two models, the Bayesian estimator for
$\mm_\varphi^2$ can be computed in polynomial time, while computing
the Bayesian estimator for $\mm_\varphi^1$ is ${\sf P}_{||}^{\sf
  NP}$-hard, which means that it is at least {\sf NP}-hard. Our
results are summarized in Table~\ref{tab:comp}. 

We also compare the asymptotic outcomes of the two rules with the Kemeny rule for winners, which is a natural extension of the {maximum likelihood estimator} of $\mm_\varphi^1$ proposed by~\citet{Fishburn77:Condorcet}.
It turns out that when $n$ votes are generated under $\mm_\varphi^1$, all three rules select the same winner asymptotically almost surely (a.a.s.) as $n\ra\infty$. When the votes are generated according to $\mm_\varphi^2$, the rule for $\mm_\varphi^1$ still selects the same winner as Kemeny a.a.s.; however, for some parameters, the winner selected by the rule for $\mm_\varphi^2$ is different with non-negligible probability. These are confirmed by experiments on synthetic datasets.

\begin{table*}[t]
\centering
\begin{tabular}{|@{\hspace{.5mm}}c@{\hspace{.5mm}}|@{\small \hspace{.5mm}}c@{\hspace{.5mm} }|@{\hspace{.5mm} }c@{\hspace{.5mm}}|@{\hspace{.5mm}}c@{\hspace{.5mm}}|@{\hspace{.5mm}}c@{\hspace{.5mm}}|@{\hspace{.5mm}}c@{\hspace{.5mm}}|}
\hline 
& \begin{tabular}{@{\small }c@{}}Anonymity, neutrality\\ \small Monotonicity
\end{tabular}& \begin{tabular}{@{\small }c@{}} Majority,\\ Condorcet\end{tabular}& \small  Consistency& \small Complexity&\small  Min.~Bayesian risk\\ 
\hline Kemeny& Y& Y& N& \begin{tabular}{@{}c@{}}{\sf NP}-hard,  ${\sf P}_{||}^{\sf NP}$-hard\end{tabular}& N\\
\hline \begin{tabular}{c}Bayesian est.~of \\  $\mm_\varphi^1$ {\footnotesize (uni.~prior)}\end{tabular} & Y& \begin{tabular}{c}N \end{tabular}& N  & \begin{tabular}{c}{\sf NP}-hard,  ${\sf P}_{||}^{\sf NP}$-hard\\ {\footnotesize (Theorem~\ref{thm:compmodel1})}\end{tabular}& Y\\
\hline \begin{tabular}{c}Bayesian est.~of \\ $\mm_\varphi^2$ {\footnotesize (uni.~prior)}\end{tabular}& Y& \begin{tabular}{c}N \end{tabular}& N  & {\sf P} {\footnotesize (Theorem~\ref{thm:compmodel2})}& Y\\
\hline
\end{tabular}\vspace{-2mm}
\caption{\small Kemeny for winners vs.~Bayesian estimators of $\mm_\varphi^1$ and $\mm_\varphi^2$ to choose {\em winners}. \label{tab:comp}}\vspace{-4mm}
\end{table*}

{\noindent \bf Related work.} Along the second perspective in social choice (to make an objectively correct decision), in addition
to Condorcet's statistical approach to social
choice~\citep{Condorcet1785:Essai,Young88:Condorcet}, most previous work in economics, political science, and statistics focused on extending the theorem to heterogeneous, correlated, or strategic agents for two alternatives, see~\cite{Nitzan84:Significance,Austen96:Information} among many others. Recent work in computer science views agents' votes as i.i.d.~samples from a statistical model, and computes the MLE to estimate the parameters that
maximize the likelihood~\citep{Conitzer05:Common,Conitzer09:Preference,Xia10:Aggregating,Xia11:Maximum,Azari12:Random,Procaccia12:Maximum,Caragiannis13:When}. A limitation of these approaches is that they estimate the
parameters of the model, but may not directly inform the right {\em 
decision} to make in the multi-agent context. 
%
The main approach has
been to return the modal rank order implied by the estimated parameters,
or the alternative with the highest, predicted marginal probability of
being ranked in the top position.

%
%
%
There have also been some proposals to go beyond MLE in social choice. In
fact, \citet{Young88:Condorcet} proposed to select a winning
alternative that is {\em ``most likely to be the best (i.e.,
  top-ranked in the true ranking)''} and provided formulas to compute it for three alternatives. This idea has been formalized and extended by~\citet{Procaccia12:Maximum} to choose a given number of alternatives with highest marginal probability under the Mallows model. More recently, independent to our work, \citet{Elkind14:How} investigated a similar question for choosing multiple winners under the Condorcet model. 
  We will see that these are special cases of our proposed framework in Example~\ref{ex:young}.~\citet{Pivato13:Voting} conducted a similar study
to~\citet{Conitzer05:Common}, examining voting rules that can be
interpreted as expect-utility maximizers.

We are not aware of previous work that frames the problem of social
choice from the viewpoint of statistical decision theory, which is our main conceptual contribution.
Technically, the approach taken in this paper advocates a general paradigm of ``design by statistics, evaluation by social choice and computer science''. We are not aware of a previous work following this paradigm to design and evaluate new rules. Moreover, the normative properties for the two voting rules investigated in this paper are novel, even though these rules are not really novel. Our result on the computational complexity of the first rule strengthens the NP-hardness result by~\citet{Procaccia12:Maximum}, and the complexity for the second rule (Theorem~\ref{thm:comp}) was independently discovered by~\citet{Elkind14:How}. 

%
%
The statistical decision-theoretic framework is quite general, 
allowing considerations such as
estimators that minimize the maximum expected loss, or the maximum expected regret~\citep{Berger85:Statistical}. 
In a different context, focused on uncertainty about the availability
of alternatives,~\citet{Lu10:Unavailable}
adopt a  decision-theoretic view of 
the design of an optimal voting rule. \citet{Caragiannis14:Modal} studied the robustness of social choice mechanisms w.r.t.~model uncertainty, and characterized a unique social choice mechanism that is consistent w.r.t.~a large class of ranking models.

A number of recent papers in computational social choice take utilitarian and decision-theoretical approaches towards social choice~\cite{Procaccia06:Distortion,Caragiannis11:Voting,Boutilier11:Probabilistic,Boutilier12:Optimal}. Most of them evaluate the joint decision w.r.t.~agents' {\em subjective} preferences, for example the sum of agents' subjective utilities (i.e.~the {\em social welfare}). We don't view this as fitting into the classical approach to {\em statistical} decision theory as formulated by~\citet{Wald50:Statistical}. In our framework, the joint decision is evaluated objectively w.r.t.~the ground truth in the statistical model. 
Several papers in machine learning developed algorithms to compute MLE or Bayesian estimators for popular ranking models~\cite{Kuo09:Learning,Long10:Active,Lu11:Learning}, but without considering the normative properties of the estimators.
\vspace{-1mm}\section{Preliminaries}\vspace{-3mm}
In social choice, we have a set of $m$ alternatives
$\mc=\{c_1,\ldots,c_m\}$ and a set of $n$ agents. Let $\ml(\mc)$
denote the set of all {linear orders} over $\mc$. For any alternative
$c$, let $\ml_c(\mc)$ denote the set of linear orders over $\mc$ where
$c$ is ranked at the top. Agent $j$ uses a {linear order}
$V_j\in\ml(\mc)$ to represent her preferences, called her {\em vote}.
The collection of agents votes is called a {\em profile}, denoted by
$P=\{V_1,\ldots,V_n\}$. A {\em (irresolute) voting rule}
$r:\ml(\mc)^n\ra (2^\mc\setminus\emptyset)$ selects a set of winners that are ``tied'' for the first place
for every profile of $n$ votes.

For any pair of linear orders $V,W$, let $\kendall(V,W)$ denote the {\em Kendall-tau distance} between $V$ and $W$, that is, the number of different pairwise comparisons in $V$ and $W$. 
The {\em Kemeny rule} (a.k.a.~{\em Kemeny-Young
  method})~\citep{Kemeny59:Mathematics,Young78:Consistent} selects all
{\em linear orders} with the minimum Kendall-tau distance from the
preference profile $P$, that is,
$\text{Kemeny}(P)=\arg\min_{W}\kendall(P,W)$. The most well-known variant of Kemeny
to select winning alternatives, denoted by $\kemeny_\mc$, is due
to~\citet{Fishburn77:Condorcet}, who defined it as a voting rule that selects all
alternatives that are ranked in the top position of some winning
linear orders under the Kemeny rule. That is, $\kemeny_\mc(P)=\{top(V):
V\in\text{Kemeny}(P)\}$, where $top(V)$ is the top-ranked alternative
in $V$.

Voting rules are often evaluated by the following  normative properties. 
An irresolute rule $r$ satisfies:

$\bullet$ {\em anonymity}, if $r$ is insensitive to permutations over agents;

$\bullet$ {\em neutrality}, if $r$ is insensitive to permutations over alternatives;

$\bullet$ {\em monotonicity}, if for any $P$, $c\in r(P)$, and any $P'$ that is obtained from $P$ by only raising the positions of $c$ in one or multiple votes, then $c\in r(P')$;

$\bullet$ {\em Condorcet criterion}, if for any profile $P$ where a Condorcet winner exists, it must be the unique winner. A Condorcet winner is the alternative that beats every other alternative in pair-wise elections.

$\bullet$ {\em majority criterion}, if for any profile $P$ where an alternative $c$ is ranked in the top positions for more than half of the votes, then $r(P)=\{c\}$. If $r$ satisfies Condorcet criterion then it also satisfies the majority criterion.

$\bullet$ {\em consistency}, if for any pair of profiles $P_1,P_2$ with $r(P_1)\cap r(P_2)\neq\emptyset$, $r(P_1\cup P_2)=r(P_1)\cap r(P_2)$.

For any profile $P$, its {\em weighted majority graph (WMG)}, denoted
by $\text{WMG}(P)$, is a weighted directed graph whose vertices are
$\mc$, and there is an edge between any pair of alternatives $(a,b)$
with weight $w_P(a,b)=\#\{V\in P:a\succ_V b\}-\#\{V\in P:b\succ_V
a\}$.

A parametric 
model $\mm=(\Theta,\ms,\Pr)$ is composed of three parts:
a {\em parameter space} $\Theta$, a {\em sample space} $\ms$ composing of all datasets, and a
set of probability distributions over $\ms$ indexed by elements of
$\Theta$: for each $\theta\in\Theta$, the distribution indexed by
$\theta$ is denoted by $\Pr(\cdot|\theta)$.\footnote{This
notation should not be taken
to mean a conditional distribution over $\ms$ unless we are taking
  a Bayesian point of view.}

Given a parametric model $\mm$, a {\em maximum likelihood estimator
  (MLE)} is a function $f_\text{MLE}:\ms\ra \Theta$ such that for any
data $P\in\ms$, $f_\text{MLE}(P)$ is a parameter that maximizes the
likelihood of the data. That is, $f_\text{MLE}(P)\in
\arg\max_{\theta\in\Theta}\Pr(P|\theta)$.

In this paper we focus on {\em parametric ranking models}.
Given $\mc$, a parametric ranking model $\mm_\mc=(\Theta,\Pr)$ is
composed of a parameter space $\Theta$ and a distribution
$\Pr(\cdot|\theta)$ over $\ml(\mc)$ for each $\theta\in\Theta$, such
that for any 
number of voters $n$, the sample space is $\ms_n=\ml(\mc)^n$, where each
vote is generated i.i.d. from $\Pr(\cdot|\theta)$. Hence, for any
profile $P\in\ms_n$ and any $\theta\in\Theta$, we have
$\Pr(P|\theta)=\prod_{V\in P}\Pr(V|\theta)$. We
omit the sample space because it is
determined by $\mc$ and $n$. 
%

\begin{dfn} In the {\em Mallows model}~\citep{Mallows57:Non-null}, a parameter is composed of a linear order $W\in \ml(\mc)$ and  a  {\em dispersion parameter} $\varphi$ with $0<\varphi<1$. For any profile $P$ and $\theta=(W,\varphi)$,  $\Pr(P|\theta)=\prod_{V\in P}\frac{1}{Z}\varphi^{\text{Kendall}(V,W)}$, where $Z$ is the normalization factor with $Z=\sum_{V\in \ml(\mc)}\varphi^{\text{Kendall}(V,W)}$. \end{dfn}

{\em Statistical decision theory}~\citep{Wald50:Statistical,Berger85:Statistical} studies
scenarios where the decision maker must make a {\em decision} $d\in\md
$ based on the data $P$ generated from a parametric model, generally
$\mm=(\Theta,\ms,\Pr)$. The quality of the decision is evaluated by a
{\em loss function} $L:\Theta\times\md \ra{\mathbb R}$, which takes
the {\em true} parameter and the decision as inputs.

In this paper, we focus on the {\em Bayesian principle} of statistical
decision theory to design social choice mechanisms as choice functions
that minimize the {\em Bayesian risk} under a prior distribution over
$\Theta$. More precisely, the Bayesian risk, $R_B(P, d)$, is the
expected loss of the decision $d$ when the parameter is generated
according to the posterior distribution given data $P$. That is, $R_B(P,
d)=E_{\theta|P}L(\theta,d)$. Given a parametric model $\mm$, a loss
function $L$, and a prior distribution over $\Theta$, a
(deterministic) {\em Bayesian estimator} $f_B$ is a decision rule that
makes a deterministic decision in $\md$ to minimize the Bayesian risk, that is, for any $P\in \ms$,
$f_B(P)\in\arg\min_{d}R_B(P, d)$. We focus on deterministic estimators
in this work and leave randomized estimators for future research.

\begin{ex}\label{ex:sdt} When $\Theta$ is discrete, an MLE of a
  parametric model $\mm$ is a Bayesian estimator of the statistical
  decision problem $(\mm,\md =\Theta,L_{0\text{-}1})$ under the
  uniform prior distribution, where $L_{0\text{-}1}$ is the {\em
    $0$-$1$ loss function} such that $L_{0\text{-}1}(\theta,d)=0$ if
  $\theta =d$, otherwise $L_{0\text{-}1}(\theta,d)=1$.
\end{ex}

In this sense, all previous MLE approaches in social choice can be
viewed as the Bayesian estimators of a statistical decision-theoretic
framework for social choice where $\md=\Theta$, a $0$-$1$ loss
function, and the uniform prior.

\section{Our Framework}\vspace{-2mm}
Our framework is quite general and flexible because we can choose any
parametric ranking model, any decision space, any loss function, and
any prior to use the Bayesian estimators social choice mechanisms.
Common choices of both $\Theta$ and $\md $ are $\ml(\mc)$, $\mc$, and
$(2^{\mc}\setminus\emptyset)$.

\begin{dfn} 
A statistical decision-theoretic framework for social choice is a tuple $\mf = (\mm_\mc,\md,L)$, where $\mc$ is the set of alternatives, $\mm_\mc=(\Theta,\Pr)$ is a parametric ranking model, $\md $ is the decision space, and $L:\Theta\times \md \ra {\mathbb R}$ is a loss function.
\end{dfn}

Let $\mb(\mc)$ denote the set of all irreflexive, antisymmetric, and
total binary relations over $\mc$. For any $c\in\mc$, let $\mb_c(\mc)$
denote the relations in $\mb(\mc)$ where $c\succ a$ for all
$a\in\mc-\{c\}$. It follows that $\ml(\mc)\subseteq \mb(\mc)$, and moreover, the
Kendall-tau distance can be defined to count the number of pairwise disagreements between elements of 
$\mb(\mc)$.  

In the rest of the paper, we focus on the following
two parametric ranking models, where the dispersion is a fixed parameter.

\begin{dfn}[Mallows model with fixed dispersion, and the Condorcet model] Let 
  $\mm_\varphi^1$ denote the {\em Mallows model with fixed dispersion}, where the parameter space is $\Theta=\ml(\mc)$ and given
  any $W\in\Theta$, $\Pr(\cdot |W)$ is $\Pr(\cdot |(W,\varphi))$ in
  the Mallows model, where $\varphi$ is fixed.

In the Condorcet model, $\mm_\varphi^2$, the parameter space is $\Theta=\mb(\mc)$. For any $W\in \Theta$ and any profile $P$, we have $\Pr(P|W)=\prod_{V\in P}\left(\frac{1}{Z}\varphi^{\text{Kendall}(V,W)}\right)$, where $Z$ is the normalization factor such that  $Z=\sum_{V\in \mb(\mc)}\varphi^{\text{Kendall}(V,W)}$, and parameter $\varphi$ is fixed.\footnote{In the Condorcet model the sample space is $\mb(\mc)^n$~\cite{Xia14:Deciphering}. We study a variant with sample space $\ml(\mc)^n$.}
\end{dfn}
$\mm_\varphi^1$ and $\mm_\varphi^2$ degenerate to 
the Condorcet model for two alternatives~\citep{Condorcet1785:Essai}. The Kemeny rule that selects a linear order is an MLE of $\mm_\varphi^1$ for any $\varphi$. 

We now formally define two statistical decision-theoretic frameworks associated with $\mm_\varphi^1$ and $\mm_\varphi^2$, which are the focus of the rest of our paper. 
\begin{dfn} \label{dfn:framework} For $\Theta=\ml(\mc)$ or $\mb(\mc)$, any $\theta\in\Theta$, and any $c\in\mc$, we define a loss function $L_{top}(\theta,c)$ such that $L_{top}(\theta,c)=0$ if for all $b\in\mc$, $c\succ b$ in $\theta$; otherwise $L_{top}(\theta,c)=1$.

Let $\mf_\varphi^1=(\mm_\varphi^1,2^\mc\setminus\emptyset,L_{top})$ and $\mf_\varphi^2= (\mm_\varphi^2,2^\mc\setminus\emptyset,L_{top})$, where for any $C\subseteq \mc$, $L_{top}(\theta,C)=\sum_{c\in C}L_{top}(\theta,c)/|C|$. Let  $f_B^1$ (respectively, $f_B^2$) denote the Bayesian estimators of  $\mf_\varphi^1$ (respectively, $\mf_\varphi^2$) under the 
uniform prior.
\end{dfn}
We note that $L_{top}$ in the above definition takes a parameter and a decision in $2^\mc\setminus\emptyset$ as inputs, which makes it different from the $0$-$1$ loss function $L_{0\text{-}1}$ that takes a pair of parameters as inputs, as the one in Example~\ref{ex:sdt}. Hence, $f_B^1$ and $f_B^2$ are {\em not} the MLEs of their respective models, as was the case in Example~\ref{ex:sdt}. We focus on voting rules obtained by our framework with $L_{top}$. Certainly our framework is not limited to  this loss function.
\begin{ex}\label{ex:young}
Bayesian estimators 
$f_B^1$ and $f_B^2$ coincide with
  \citet{Young88:Condorcet}'s idea of selecting the alternative that is
  {\em ``most likely to be the best (i.e., top-ranked in the true
    ranking)''}, under $\mf_\varphi^1$ and $\mf_\varphi^2$
respectively. This gives a theoretical justification of Young's idea and other followups under our framework. Specifically, $f_B^1$ is similar to rule studied by~\citet{Procaccia12:Maximum} and $f_B^2$ was independently studied by~\citet{Elkind14:How}.
\end{ex}

The following lemma provides a convenient way to compute the
likelihood in $\mm_\varphi^1$ and $\mm_\varphi^2$ from the WMG.

\begin{lemma}\label{lem:wmg} In $\mm_\varphi^1$ (respectively,  $\mm_\varphi^2$), for any $W\in\ml(\mc)$ (respectively, $W\in\mb(\mc)$) and any  profile $P$, $\Pr(P|W)\propto \prod_{c\succ_W b}\varphi^{-w_P(c,b)/2}$.\end{lemma}
\begin{proof} For any $c\succ_Wb$, the number of times $b\succ c$ in $P$ is $(n-w_P(c,b))/2$, which means that $\Pr(P|W)=\varphi^{n^2(n-1)/4} \prod_{c\succ_W b }\varphi^{-w_P(c,b)/2}$.
\end{proof}

\section{Normative Properties of Bayesian Estimators}

In this section, we compare $f_B^1$, $f_B^2$, and the Kemeny rule (for
alternatives) w.r.t.~various normative properties. We will frequently
use the following lemma, whose proof follows directly from Bayes' rule. We recall that
$\ml_c(\mc)$ is the set of all linear orders where $c$ is ranked in
the top, and $\mb_c(\mc)$ is the set of binary relations in $\mb(\mc)$
where $c$ is ranked in the top.

\begin{lemma}\label{lem:cal} In $\mf_\varphi^1$ under
the uniform prior, for any profile $P$ and any $c,b\in\mc$, $R_B(P,c)\leq R_B(P,b)$ if and only if $\sum_{V\in\ml_c(\mc)}\Pr(P|V)\geq \sum_{V\in\ml_b(\mc)}\Pr(P|V)$.

In $\mf_\varphi^2$ under
the uniform prior, for any profile $P$ and any $c,b\in\mc$, $R_B(P,c)\leq R_B(P,b)$ if and only if $\sum_{V\in\mb_c(\mc)}\Pr(P|V)\geq \sum_{V\in\mb_b(\mc)}\Pr(P|V)$.
\end{lemma}

\begin{thm}\label{thm:normative1} For any $\varphi$, $f_B^1$ satisfies anonymity, neutrality, and monotonicity. It does not satisfy majority or
the Condorcet criterion for all $\varphi>\frac{1}{\sqrt 2}$,\footnote{Whether $f_B^1$ satisfies majority and Condorcet criterion for $\varphi\leq\frac{1}{\sqrt 2}$ is an open question.} and it does not satisfy consistency.
\end{thm} 
\begin{proof} Anonymity and neutrality are obviously satisfied. 

\noindent{\bf Monotonicity.} Suppose $c\in f_B^1(P)$. To prove that $f_B^1$ satisfies monotonicity, it suffices to prove that for any profile $P'$ obtained from $P$ by raising the position of $c$ in one vote, $c\in f_B^1(P')$. We first prove the following lemma.

\begin{lemma}\label{lem:raise} For any $c\in \mc$, let  $P'$ denote a profile obtained from $P$ by raising the position of $c$ in one vote.
For any $W\in \ml_c(\mc)$, $\pr(P'|W)=\pr(P|W)/\varphi$; for any $b\in\mc$ and any $V\in  \ml_b(\mc)$, $\pr(P'|V)\leq \pr(P|V)/\varphi$. For any $W'\in \mb_c(\ml)$, $\pr(P'|W')\leq \pr(P|W')/\varphi$; for any $b\in\mc$ and any $V'\in  \mb_b(\mc)$, $\pr(P'|V')\leq \pr(P|V')/\varphi$. 
 \end{lemma}
\begin{proof} For $W\in \ml_c(\mc)$, the lemma holds because $\kendall(P',W)=\kendall(P,W)-1$, and for $V\in  \ml_b(\mc)$, the lemma holds because $\kendall(P',V)\geq \kendall(P,V)-1$. The proof for $\mb_c$ and $\mb_b$ is similar.
\end{proof} 

Therefore, for any $b\neq c$,  by Lemma~\ref{lem:raise}, we have $\sum_{W\in\ml_c(\mc)}\Pr(P'|W)=\sum_{W\in\ml_c(\mc)}\Pr(P|W)/\varphi\geq \sum_{V\in\ml_b(\mc)}\Pr(P|V)/\varphi\geq \sum_{V\in\ml_b(\mc)}\Pr(P'|V)$, which proves that $c\in f_B^1(P')$ following Lemma~\ref{lem:cal}.

\noindent{\bf Majority and the Condorcet criterion.} Let $\mc=\{c,b,c_3,\ldots,c_m\}$. We construct a profile $P^*$ where $c$ is ranked in the top positions for more than half of the votes, which means that $c$ is the Condorcet winner, but $c\not\in f_B^1(P)$.

For any $k$, let $P^*$ denote a profile composed of  $k+1$ copies of $[c\succ b\succ c_3\succ \cdots \succ c_m]$ and  $k-1$ copies of $[ b\succ c_3\succ \cdots \succ c_m\succ c]$. It is not hard to verify that the WMG of $P^*$ is as in Figure~\ref{fig:condorcet}.\begin{figure}[htp]
\centering
\includegraphics[trim=0cm 15.5cm 15cm 0cm, clip=true, width=.4\textwidth]{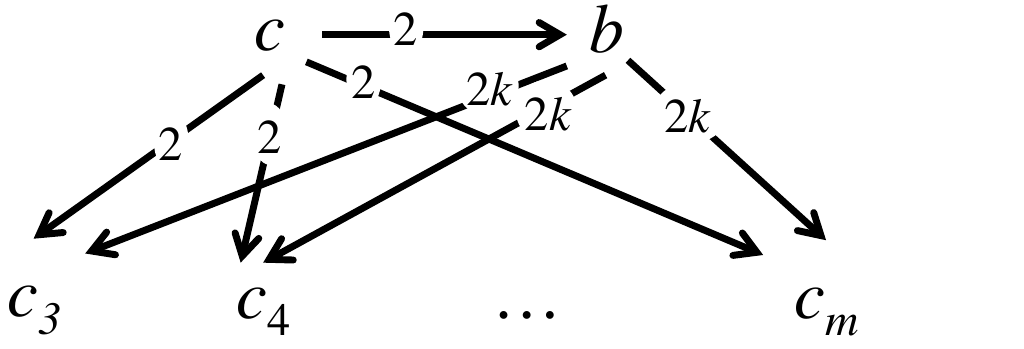}
\caption{\small The WMG of the profile $P^*$ where only positive edges are shown.\label{fig:condorcet}}
\end{figure}
\begin{lemma}\label{lem:ratio} $\frac{\sum_{V\in\ml_c(\mc)}\Pr(P^*|V)}{\sum_{W\in\ml_b(\mc)}\Pr(P^*|W)}=\frac{1+\varphi^{2k}+\cdots+\varphi^{2k(m-2)}}{1+\varphi^{2}+\cdots+\varphi^{2(m-2)}}\cdot\frac{1}{\varphi^2}$
\end{lemma}
\begin{proof} Let $\ml_{-c}=\ml-\{c\}$ and let $P|_{-c}$ denote the profile where $c$ is removed from all rankings.
\begin{align}
\sum_{V\in\ml_c(\mc)}\Pr(P^*|V)
\propto &\varphi^{-m+1}\sum_{V'\in\ml(\mc_{-c})}\varphi^{\kendall(P|_{{-c}},V')}\notag\\
\propto &\varphi^{-m+1}\sum_{V'\in\ml(\mc_{-c})}\prod_{a,d\in\mc-\{c\}: a\succ_{V'} d}\varphi^{-w_{P^*}(a,d)/2}\notag\\
\propto &\varphi^{-m+1}\sum_{t=0}^{m-2}{m-2\choose t}t!(m-2-t)!\varphi^{kt}\varphi^{-k(m-2-t)}\label{equ:comb}\\
\propto &\varphi^{-(m-2)k-m+1}\sum_{t=0}^{m-2}\varphi^{2kt}\notag
\end{align}

In (\ref{equ:comb}), $t$ is the number of alternatives in $\{c_3,\ldots,c_m\}$ ranked above $b$ in  $V'$. There are ${m-2\choose t}$ such combinations, for each of which there are $t!$ rankings among alternatives ranked above $b$ and $(m-2-t)!$ rankings among alternatives ranked below $t$. Notice that there are no edges between alternatives in $\mc-\{c,b\}$ in the WMG, which means that for any $V'$ where exactly $t$ alternatives are ranked above $b$, the probability is proportional to $\varphi^{kt}\varphi^{-k(m-2-t)}$ by Lemma~\ref{lem:wmg}. 
Similarly, $\sum_{V\in\ml_b(\mc)}\Pr(P^*|V)\propto\varphi^{-k(m-2)+1-(m-2)}\sum_{t=0}^{m-2}\varphi^{2t}$.
\end{proof}
\smallskip

Since
$\lim_{m\ra\infty}\lim_{k\ra\infty}\frac{1+\varphi^{2k}+\cdots+\varphi^{2k(m-2)}}{1+\varphi^{2}+\cdots+\varphi^{2(m-2)}}\cdot\frac{1}{\varphi^2}=\frac{1-\varphi^2}{\varphi^2}$,
for any $\varphi>\frac{1}{\sqrt 2}$, we can choose $m$ and $k$ so that
$\frac{\sum_{V\in\ml_c(\mc)}\Pr(P|V)}{\sum_{W\in\ml_b(\mc)}\Pr(P|W)}<1$. By
Lemma~\ref{lem:ratio}, $c$ is the
Condorcet winner in $P^*$ but it does not minimize the Bayesian risk
under $\mm_\varphi^1$, which means that it is not a winner under
$f_B^1$.

\noindent{\bf Consistency.} We construct an example to show that $f_B^1$ does not satisfy consistency.  In our construction $m$ and $n$ are even, and $\mc=\{c,b,c_3,c_4\}$. Let $P_1$ and $P_2$ denote profiles whose WMGs are as shown in Figure~\ref{fig:consistency}, respectively.
%
\begin{figure}[htp]
\centering
\begin{tabular}{ccc}
\includegraphics[trim=0cm 15.5cm 19cm 0cm, clip=true, width=.2\textwidth]{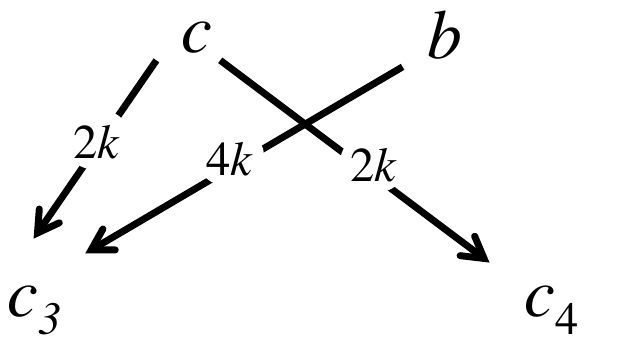}&\includegraphics[trim=0cm 15.5cm 19cm 0cm, clip=true, width=.2\textwidth]{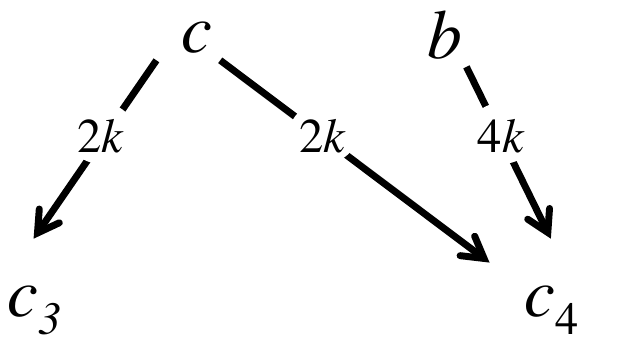}&\includegraphics[trim=0cm 15.5cm 19cm 0cm, clip=true, width=.2\textwidth]{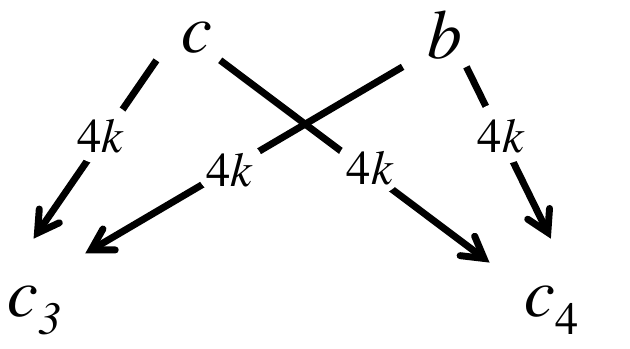}\\
$P_1$&${P_2}$&$P_1\cup P_2$
\end{tabular}
\caption{\small The WMGs of $P_1$, $P_2$, and $P_1\cup P_2$. Only positive edges are shown.\label{fig:consistency}}
\end{figure}

We provide the following lemma to compare the Bayesian risk of $c$ and $d$.
The proof is similar to the proof of Lemma~\ref{lem:ratio}.
\begin{lemma}\label{lem:consistency} Let $P\in\{P_1,P_2\}$, $\frac{\sum_{V\in\ml_c(\mc)}\Pr(P|V)}{\sum_{W\in\ml_b(\mc)}\Pr(P|W)}=\frac{3(1+\varphi^{4k})}{2(1+\varphi^{2k}+\varphi^{4k})}$
\end{lemma}
\begin{proof} Let $P=P_1$ or $P_2$.
\begin{align}
\sum_{V\in\ml_c(\mc)}\Pr(P|V)
\propto &\varphi^{-2k}\sum_{V'\in\ml(\mc_{-c})}\varphi^{\kendall(P|_{\mc_{-c}},V')}\notag\\
\propto &\varphi^{-2k}3(\varphi^{-2k}+\varphi^{2k})\notag
\end{align}
Similarly $\sum_{V\in\ml_b(\mc)}\Pr(P|V)\propto\varphi^{-2k}2(\varphi^{-2k}+1+\varphi^{2k})$. \end{proof}

For any $0<\varphi<1$, $\frac{3(1+\varphi^{4k})}{2(1+\varphi^{2k}+\varphi^{4k})}>1$ for all $k$. It is not hard to verify that  $f_B^1(P_1)=f_B^1(P_2)=\{c\}$. However, it is not hard to verify that $f_B^1(P_1\cup P_2)=\{c,b\}$, which means that $f_B^1$ is not consistent.
This completes the proof of the theorem.
\end{proof}

\begin{thm}\label{thm:normative2} For any $\varphi$, $f_B^2$ satisfies anonymity, neutrality, and monotonicity. It does not satisfy majority,
the Condorcet criterion, or consistency.
\end{thm} 
\begin{proof} 
Anonymity and neutrality are obvious. The proof for monotonicity is similar to the proof for $f_B^1$ and uses the second part of Lemma~\ref{lem:raise}.

{\noindent\bf Majority and Condorcet criterion.} We prove that $f_B^2$ does not satisfy majority or the Condorcet criterion for the same profile $P^*$ as used
in the proof of Theorem~\ref{thm:normative1}.  By Theorem~\ref{thm:comp} in the next section, we have:
\begin{align}
\frac{\sum_{V\in\mb_c(\mc)}\Pr(P^*|V)}{\sum_{W\in\mb_b(\mc)}\Pr(P^*|W)}=\frac{(\frac{1}{1+\varphi^2})^{m-1}}{(\frac{1}{1+\varphi^{2k}})^{m-2}(\frac{1}{1+\varphi^{-2}})}
=(\frac{1+\varphi^{2k}}{1+\varphi^2})^{m-2}\cdot\frac{1+\varphi^{-2}}{1+\varphi^{2}}\label{equ:f2condorcet}
\end{align}
For any $k\ge 2$, there exits $m$ such that (\ref{equ:f2condorcet})$<1$, which means that the Condorcet winner $c$ is not in $f_B^2(P^*)$.

{\noindent\bf Consistency.} We use the same profiles $P_1$ and $P_2$ as in the proof of Theorem~\ref{thm:normative1} (see Figure~\ref{fig:consistency}). For $P=P_1$ or $P_2$, we have:
\begin{align}
&\frac{\sum_{V\in\mb_c(\mc)}\Pr(P|V)}{\sum_{W\in\mb_b(\mc)}\Pr(P|W)}=\frac{(\frac{1}{1+1})(\frac{1}{1+\varphi^{2k}})^{2}}{(\frac{1}{1+1})^2(\frac{1}{1+\varphi^{4k}})}
=\frac{2(1+\varphi^{4k})}{(1+\varphi^{2k})^2}\label{equ:f2consistency}
\end{align}
For any $k$ and $m$, we have
that the value of (\ref{equ:f2consistency}) is strictly
greater than 1. 
It is not hard to verify that  $f_B^2(P_1)=f_B^2(P_2)=\{c\}$
and $f_B^2(P_1\cup P_2)=\{c,d\}$, which means that $f_B^2$ is not consistent.
\end{proof}

By Theorem~\ref{thm:normative1} and~\ref{thm:normative2}, $f_B^1$ and $f_B^2$  do not satisfy as many desired normative properties as the Kemeny rule (for winners). On the other hand, they minimize Bayesian risk under $\mf_\varphi^1$ and $\mf_\varphi^2$, respectively, for which Kemeny does neither. In addition, neither $f_B^1$ nor $f_B^2$ satisfy consistency, which means that they are not positional scoring rules.

\section{Computational Complexity}

We consider the following two types of 
decision problems.
\begin{dfn} In the {\sc better Bayesian decision} problem for a statistical decision-theoretic framework $(\mm_\mc,\md,L)$ under a prior distribution, we are given $d_1,d_2\in\md$, and a profile $P$. We are asked whether $R_B(P,d_1)\leq R_B(P,d_2)$.
\end{dfn}

We are also interested in checking whether a given alternative is the
optimal decision.
\begin{dfn} In the {\sc optimal Bayesian decision} problem  for a statistical decision-theoretic framework $(\mm_\mc,\md,L)$ under a prior distribution, we are given $d\in\md$ and a profile $P$. We are asked whether $d$ minimizes the Bayesian risk $R_B(P,\cdot)$.
\end{dfn}

${\sf P}_{||}^{\sf NP}$ is the class of decision problems that can be computed by a {\em P} oracle machine with polynomial number of parallel calls to an {\sf NP} oracle. A decision problem $A$ is ${\sf P}_{||}^{\sf NP}$-hard, if for any ${\sf P}_{||}^{\sf NP}$ problem $B$, there exists a polynomial-time many-one reduction from $B$ to $A$. It is known that ${\sf P}_{||}^{\sf NP}$-hard problems are ${\sf NP}$-hard.
\begin{thm}\label{thm:compmodel1} For any $\varphi$, {\sc better Bayesian decision} and {\sc optimal Bayesian decision} for $\mf_\varphi^1$ under uniform prior are ${\sf P}_{||}^{\sf NP}$-hard.
\end{thm}
\begin{proof} The hardness of both problems is proved by a unified polynomial-time many-one reduction from the {\sc Kemeny winner} problem, which was proved to be ${\sf P}_{||}^{\sf NP}$-complete by~\citet{Hemaspaandra05:Complexity}. In a {\sc Kemeny winner} instance, we are given a profile and an alternative $c$, and we are asked if $c$ is ranked in the top of at least one $V\in\ml(\mc)$ that minimizes $\kendall(P,V)$.

For any alternative $c$, the {\em Kemeny score} of $c$ under $\mm_\varphi^1$ is the smallest distance between the profile $P$ and any linear order where $c$ is ranked in the top. We prove that when $\varphi<\frac{1}{m!}$, the Bayesian risk of $c$ is largely determined by the Kemeny score of $c$:
\begin{lemma}\label{lem:kemeny} For any $\varphi<\frac{1}{m!}$ and $c,b\in\mc$, if the Kemeny score of $c$ is strictly smaller than the Kemeny score of $b$, then $R_B(P,c)<R_B(P,b)$ for $\mm_\varphi^1$.
\end{lemma}
\begin{proof} Let $k_c$ and $k_b$ denote the Kemeny scores of $c$ and $b$, respectively. We have $\sum_{V\in\ml_c(\mc)}\Pr(P|V)>\frac{1}{Z^n}\varphi^{k_c}>\frac{1}{Z^n}m!\varphi^{k_c-1}\geq \sum_{V\in\ml_b(\mc)}\Pr(P|V)$, which means that $R_B(P,c)<R_B(P,b)$ by Lemma~\ref{lem:cal}.
\end{proof}

We note that $\varphi$ may be larger than $\frac{1}{m!}$. In our reduction, we will duplicate the input profile so that effectively we are computing the problems for a small $\varphi$. Let $t$ be any natural number such that $\varphi^t<\frac{1}{m!}$. For any {\sc Kemeny winner} instance $(P,c)$ for alternatives $\mc'$, we add two more alternatives $\{a,b\}$ and define a profile $P'$ whose WMG is as shown in Figure~\ref{fig:hardness} using  McGarvey's trick~\citep{McGarvey53:Theorem}. The WMG of $P'$ contains the $\text{WMG}(P)$ as a subgraph, where the weights are $6$ times of the weights of $\text{WMG}(P)$; for all $c'\in \mc'$, the weight of $a\ra c'$ is $6$; for all $c'\in\mc'-\{c\}$, the weight of $b\ra c'$ is $6$; the weight of $c\ra b$ is $4$ and the weight of $b\ra a$ is $2$.
\begin{figure}[htp]
\centering
\includegraphics[trim=0cm 15.5cm 18.5cm 0cm, clip=true, width=.23\textwidth]{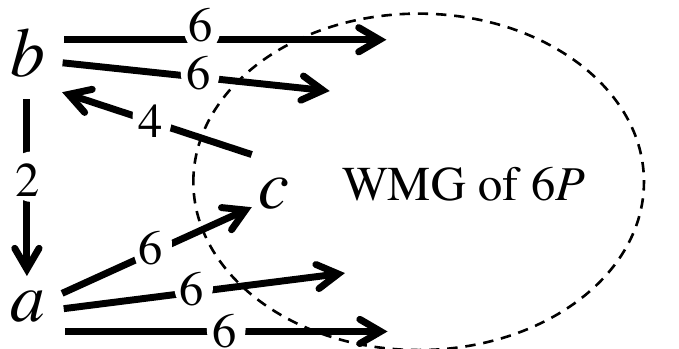}
\caption{\small The WMG of $P'$. $P^*=tP'$.\label{fig:hardness}}
\end{figure}

Then, we let $P^*=tP$, which is $t$ copies of $P$. It follows that for any $V\in\ml(\mc)$, $\Pr(P^*|V,\varphi)=\Pr(P'|V,\varphi^t)$. By Lemma~\ref{lem:kemeny}, if an alternative $e$ has the strictly lowest Kemeny score for profile $P'$, then it the unique alternative that minimizes the Bayesian risk for $P'$ and dispersion parameter $\varphi^t$, which means that $e$ minimizes the Bayesian risk for $P^*$ and dispersion parameter $\varphi$. 

Let $O$ denote the set of linear orders over $\mc'$ that minimizes the Kendall tau distance from $P$ and let $k$ denote this minimum distance. Choose an arbitrary $V'\in O$. Let $V=[b\succ a\succ V']$. It follows that $\kendall(P',V)=4+6k$. If there exists $W'\in O$ where $c$ is ranked in the top position, then we let $W=[a\succ c\succ b\succ (V'-\{c\})]$. We have $\kendall(P',W)=2+6k$. If $c$ is not a Kemeny winner in $P$, then for any $W$ where $b$ is not ranked in the top position, $\kendall(P',W)\geq 6+6k$. Therefore, $a$ minimizes the Bayesian risk if and only if $c$ is a Kemeny winner in $P$, and if $c$ does not minimizes the Bayesian risk, then $b$ does. Hence {\sc better decision} (checking if $a$ is better than $b$) and {\sc optimal Bayesian decision} (checking if $a$ is the optimal alternative) are ${\sf P}_{||}^{\sf NP}$-hard.
\end{proof}
\smallskip

We note that the {\sc optimal Bayesian decision} for the framework in
Theorem~\ref{thm:compmodel1} is equivalent to checking whether a given
alternative $c$ is in $f_B^1(P)$. We do not know whether these
problems are ${\sf P}_{||}^{\sf NP}$-complete.
\begin{thm}\label{thm:compmodel2} For any rational number $\varphi$,\footnote{We require $\varphi$ to be rational to avoid representational issues.}  {\sc better Bayesian decision} and {\sc optimal Bayesian decision} for $\mf_\varphi^2$ under uniform prior are in ${\sf P}$.
\end{thm}
The theorem is a corollary of the following stronger theorem that provides a closed-form formula for Bayesian loss for $\mf_\varphi^2$.\footnote{The formula resembles Young's calculation for three alternatives~\cite{Young88:Condorcet}, where it was not clear whether the calculation was done for $\mf_\varphi^2$. Recently it was clarified by~\citet{Xia14:Deciphering} that this is indeed the case.} We recall that for any profile $P$ and any pair of alternatives $c,b$, $w_P(c,b)$ is the weight on $c\ra b$ in the weighted majority graph of $P$.

\begin{thm}\label{thm:comp} For $\mf_\varphi^2$ under uniform prior, for any $c\in\mc$, 
$R_B(P,c)=1-\prod_{b\neq c}\dfrac{1}{1+\varphi^{w_P(c,b)}}$.
\end{thm}
\begin{proof} Given a profile $P$, for any $c,b\in\mc$, we let $P(c\succ b)$ denote the number of times $c$ is preferred to $b$ in $P$. For any $c,b\in \mc$, let $K_{\{c,b\}}=\varphi^{P(c\succ b)}+\varphi^{P(b\succ c)}$. The theorem is equivalent to proving that $\sum_{V\in \mb_c(\mc)}\Pr(V|P)=\prod_{b\neq c}\dfrac{\varphi^{P(b\succ c)}}{K_{\{c,b\}}}$.  We first calculate $\pr(P)$. 
\begin{align*}
&\pr(P)=\sum_{W\in\mb_c(\mc)}\Pr(P|W)\cdot\Pr(W)\\
=&\pr(W)\cdot\frac{1}{Z^n}\cdot\prod_{\{c,b\}}(\varphi^{P(c\succ b)}+\varphi^{P(b\succ c)})\\
=&\pr(W)\cdot \frac{1}{Z^n}\cdot\prod_{\{c,b\}}K_{\{c,b\}}
\end{align*}
For any $c\in\mc$, we have:
\begin{align*}
\sum_{W\in \mb_c(\mc)}\pr(W|P)
=&\sum_{W\in \mb_c(\mc)}\pr(P|W)\cdot\frac{\pr(W)}{\pr(P)}\\
=&\frac{\pr(W)}{\pr(P)}\cdot\frac{1}{Z^n}\cdot \prod_{b\neq c}\varphi^{P(b\succ c)}\sum_{V'\in\mb(\mc-\{c\})}\varphi^{\kendall(P|_{\mc_{-c}},V')}\\
=&\frac{\pr(W)}{\pr(P)}\cdot\frac{1}{Z^n}\cdot \prod_{b\neq c}\varphi^{P(b\succ c)}\prod_{b,e\neq c}(\varphi^{P(e\succ b)}+\varphi^{P(b\succ e)})
=\prod_{b\neq c}\dfrac{\varphi^{P(b\succ c)}}{K_{\{c,b\}}}
\end{align*}
\end{proof}

The comparisons of Kemeny, $f_B^1$, and $f_B^2$ are summarized in Table~\ref{tab:comp}. According to the criteria we considered, none of the three outperforms the others. Kemeny does well in normative properties, but does not minimize Bayesian risk under either $\mf_\varphi^1$ or $\mf_\varphi^2$, and is hard to compute. $f_B^1$ minimizes the Bayesian risk under  $\mf_\varphi^1$, but is hard to compute. We would like to highlight $f_B^2$, which minimizes the Bayesian risk under  $\mf_\varphi^2$, and more importantly, can be computed in polynomial time despite the similarity between $\mf_\varphi^1$ and $\mf_\varphi^2$. This makes $f_B^2$ a practical voting rule that is also justified by Condorcet's model.

\section{Asymptotic Comparisons}

In this section, we ask the following question: as
the number of voters, $n\ra\infty$,
what is the probability that Kemeny, $f_B^1$, and $f_B^2$ choose
different winners? 

We show that when the data is generated from $\mm_\varphi^1$, all
three methods are equal {\em asymptotically almost surely (a.a.s.)},
that is, they are equal with probability $1$ as $n\ra\infty$.
\begin{thm}\label{thm:asymptotic1}
Let $P_n$ denote a profile of $n$ votes generated i.i.d.~from $\mm_\varphi^1$ given $W\in\ml_c(\mc)$. Then, $\Pr_{n\ra\infty}(\text{Kemeny}(P_n)=f_B^1(P_n)=f_B^2(P_n)=c)=1$.
\end{thm}
\begin{sketch} It is not hard to see that asymptotically almost surely, for any pair of alternatives $a,b\in\mc$, the number of times $a\succ b$ in $P_n$ is $(1+o(1))n\Pr(a\succ b|W)$. As a corollary of a stronger theorem by~\cite{Caragiannis13:When}, as $n\ra\infty$, $c$ is the Condorcet winner, which means that $\Pr_{n\ra\infty}(\text{Kemeny}(P_n)=c)=1$. 

We now prove a lemma that will be useful for $f_B^1$ and $f_B^2$.
\begin{lemma}
\label{lem:comparison} For any $W\in\ml_c(\mc)$, any alternatives $a,b$ that are different from $c$, $\Pr(c\succ b|W)>\Pr(a\succ b|W)$.
\end{lemma}
\begin{proof}
We have $\Pr(c\succ b|W)-\Pr(a\succ b|W)=\Pr(c\succ b\succ a|W)-\Pr(a\succ b\succ c|W)$.
For any linear order $V_{c\succ b\succ a}$ where $c\succ b\succ a$, we let $V_{a\succ b\succ c}$ denote the linear order obtained from $V_{c\succ b\succ a}$ by switching the positions of $c$ and $a$. It follows that $\kendall(V_{c\succ b\succ a},W)< \kendall(V_{a\succ b\succ c},W)$, which means that $\Pr({c\succ b}|W)> \Pr({a\succ b}|W)$.
\end{proof}

To prove the theorem for $f_B^1$, it suffices to prove that for any $b\neq c$ and any $0<\varphi<1$, asymptotically almost surely, we have $\sum_{V\in \ml_c(\mc)}\varphi^{\kendall(P_n,V)}>\sum_{V\in \ml_b(\mc)}\varphi^{\kendall(P_n,V)}$. For any $V_c\in \ml_c(\mc)$, we let $V_b$ denote the linear order obtained from $V_c$ by exchanging the positions of $c$ and $b$, which means that $V_b\in\ml_b(\mc)$.
\begin{lemma}\label{lem:limitkendall}$\Pr_{n\ra\infty}(\kendall(P_n,V_c)<\kendall(P_n,V_b))=1$.
\end{lemma} 
\begin{proof} Given $V_c$, let $\mc'$ denote the set of alternatives between $c$ and $b$ in $V_c$. We have $\kendall(P_n,V_b)-\kendall(P_n,V_c)=\sum_{a\in\mc'}[w_{P_n}(a, b)-w_{P_n}(a, c)]+w_{P_n}(c, b)=\sum_{a\in\mc'}2n[\Pr(a\succ b|W)-\Pr(a\succ c|W)]+n(2\Pr(c\succ b|W)-1)+o(n)$, where we recall that $w_{P_n}(a\succ b)=P_n(a\succ b)-P_n(b\succ a)$. By Lemma~\ref{lem:comparison}, for all $a$ that is different from $b$ and $c$, $\Pr(c\succ a|W)>\Pr(b\succ a|W)$, which means $\Pr(a\succ b|W)-\Pr(a\succ c|W)>0$. Since $c$ is the Condorcet winner asymptotically almost sure, $\Pr(c\succ b|W)>1/2$. This proofs the claim.
\end{proof}

By Lemma~\ref{lem:limitkendall}, $\Pr_{n\ra\infty}(\forall V_c\in\ml_c(\mc),\kendall(P_n,V_c)<\kendall(P_n,V_d))=1$, which means that 
$$\Pr_{ n\ra\infty}\left(\forall V_c\in\ml_c(\mc),\varphi^{\kendall(P_n,V_c)}<\varphi^{\kendall(P_n,V_d)}\right)=1$$ Hence, $\Pr_{ n\ra\infty}(\sum_{V\in \ml_c(\mc)}\varphi^{\kendall(P_n,V)}>\sum_{V\in \ml_d(\mc)}\varphi^{\kendall(P_n,V)})=1$. This proves the theorem for $f_B^1$.

We use Theorem~\ref{thm:comp} and Lemma~\ref{lem:comparison} to prove
the theorem for $f_B^2$. We note that $\frac{\varphi^{P_n(b\succ
    c)}}{K_{\{c,b\}}}=\frac{1}{1+\varphi^{P_n(c\succ b)-P_n(b\succ
    c)}}=\frac{1}{1+\varphi^{2P_n(c\succ b)-n}}$. By
Lemma~\ref{lem:comparison}, $\Pr(c\succ b|W)>\Pr(a\succ b|W)$, which
means that asymptotically almost surely, we have the following steps
of reasoning:

\noindent(1) $P_n(c\succ b)>P_n(a\succ b)$ for all $a,b$.

\noindent(2) $\frac{1}{1+\varphi^{2P_n(c\succ b)-n}}>\frac{1}{1+\varphi^{2P_n(a\succ b)-n}}$ for all $a$ and $b$.

\noindent(3) $\frac{\varphi^{P_n(b\succ c)}}{K_{\{c,b\}}}>\frac{\varphi^{P_n(b\succ a)}}{K_{\{a,b\}}}$.

\noindent(4) For any $a\neq c$, $\prod_{b\neq c}\frac{\varphi^{P_n(b\succ c)}}{K_{\{c,b\}}}>\prod_{b\neq a}\frac{\varphi^{P_n(b\succ a)}}{K_{\{a,b\}}}$. 

Finally, applying Theorem~\ref{thm:comp} to (4), $c$ is the unique winner asymptotically almost surely.  This completes the proof of the theorem.
\end{sketch}

\begin{thm}\label{thm:asymptotic2} For any $W\in\mb(\mc)$ and any
  $\varphi$, $f_B^1(P_n)=\kemeny(P_n)$ a.a.s.~as $n\ra\infty$ and
  votes in $P_n$ are generated i.i.d.~from $\mm_\varphi^2$ given $W$.

For any $m\geq 5$, there exists $W\in\mb(\mc)$ such that for any $\varphi$, there exists $\epsilon>0$ such that  with probability at least $\epsilon$, $f_B^1(P_n)\neq f_B^2(P_n)$ and $\kemeny(P_n)\neq f_B^2(P_n)$~as $n\ra\infty$ and votes in $P_n$ are generated i.i.d.~from $\mm_\varphi^2$ given $W$.
\end{thm}
\begin{sketch} Due to the Central Limit Theorem, for any $V,W\in \mb(\mc)$, $|\kendall(P_n,V)-\kendall(P_n,W)|=\Omega(\sqrt n)$ a.a.s. By Lemma~\ref{lem:wmg} and Lemma~\ref{lem:cal}, any $f_B^1$ winner $c$ maximizes $\sum_{V_c\in\ml_c(\mc)}\varphi^{\kendall(P_n,V_c)}\approx \max_{V_c\in\ml_c(\mc)}\varphi^{\kendall(P_n,V_c)}$ a.a.s. This means that $c$ is the Kemeny winner a.a.s.

For the second part, we sketch a proof for $m=5$. Other cases can be proved similarly. Let $W$ denote the binary relation as shown in Figure~\ref{fig:asymptotic2}.
\begin{figure}[htp]
\centering
\includegraphics[trim=0 13cm 19cm 0, clip=true, width=.16\textwidth]{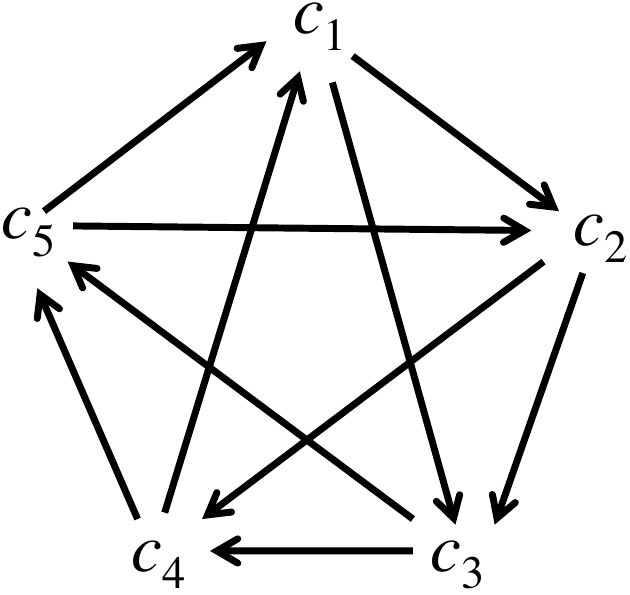}
\caption{$W\in\mb(\mc)$ for $m=5$.\label{fig:asymptotic2}}
\end{figure}

It can be verified that for all $i\leq 5$, $\Pr(c_i\succ c_{i+1}|W)$ (we let $c_1 = c_6$) are the same and are larger than $1/2$, denoted by $p_1$; for all $i\leq 5$, $\Pr(c_i\succ c_{i+2}|W)$ are the same and are larger than $1/2$, denoted by $p_2$. We define a random variable $X_{c\succ b}$ for any $c\succ_W b$ such that for any  $V\in\ml(\mc)$, if $c\succ_V b$ then $X_{c\succ b}=1$ otherwise $X_{c\succ b}=-1$.
\begin{lemma}\label{lem:noncorr} $\{X_{c\succ b}:c\succ_W b\}$ are not linearly correlated.
\end{lemma}
\begin{proof} Suppose for the sake of contradiction $\{X_{c\succ b}:c\succ_W b\}$ are linearly correlated. For any $X_{c\succ b}$ whose coefficient is non-zero, there exists a linear order $V$ where $c$ and $b$ are ranked adjacently. Let $V'$ denote the linear order obtained from $V$ by switching the positions of $c$ and $b$. We note thta $X_{c\succ b}(V)=-X_{c\succ b}(V')$, and other random variables in $\{X_{c\succ b}:c\succ_W b\}$ take the same values at $V$ and $V'$, this leads to a contradiction.
\end{proof}

Then, it follows from the multivariate Lindeberg-L\'evy Central Limit Theorem (CLT)~\citep[Theorem D.18A]{Greene11:Econometric} that $\{(\sum_{j=1}^nX_{c\succ b}-pn)/\sqrt n:c\succ_W b\}$ converges in distribution to a multivariate normal distribution $\mn(0,\Sigma)$, where $\Sigma$ is the covariance matrix, and is non-singular by Lemma~\ref{lem:noncorr}. We note that $\sum_{j=1}^nX_{c\succ b}=P_n(c\succ b)$. 

Hence, with positive probability the following hold at the same time in $\text{WMG}(P_n)$:

\noindent $\bullet$ $0<w_{P_n}(c_5,c_1)-(2p_1-1)n<\sqrt n$; $0<w_{P_n}(c_4,c_1)-(2p_2-1)n<\sqrt n$.

\noindent $\bullet$ $\sqrt n<w_{P_n}(c_1,c_2)-(2p_1-1)n<2\sqrt n$; $\sqrt n<w_{P_n}(c_5,c_2)-(2p_2-1)n<2\sqrt n$; $0<w_{P_n}(c_1,c_3)-(2p_2-1)n<\sqrt n$.

\noindent $\bullet$ For any other $c_i\succ_Wc_j$ not mentioned above, $5\sqrt n<w_{P_n}(c_i,c_j)-(2\Pr(c_i\succ c_j|W)-1)n$.

If $P_n$ satisfies all above conditions, then by Theorem~\ref{thm:comp} $f_B^2(P_n)=\{c_1\}$. Meanwhile, $\kemeny(P_n)=f_B^1(P_n)=\{c_2\}$ with $[c_2\succ c_3\succ c_4\succ c_5\succ c_1]$ minimizing the total Kendall-tau distance. This shows that $f_B^2(P_n)\neq \kemeny(P_n)$ with non-negligible probability as $n\ra\infty$,
and completes the proof of the theorem.
\end{sketch}

Theorem~\ref{thm:asymptotic1} suggests that, when $n$ is large and the votes are generated from $\mm_\varphi^1$, all of $f_B^1$, $f_B^2$, and Kemeny will choose the alternative ranked in the top of the ground truth as the winner. Similar observations have been made for other voting rules by~\cite{Caragiannis13:When}. On the other hand, Theorem~\ref{thm:asymptotic2} tells us that when the votes are generated from $\mm_\varphi^2$, interestingly, for some ground truth parameter $f_B^2$ is different from the other two with non-negligible probability, and as we will see in the next subsection, we are very confident that such probability is quite large (about $30\%$ for given $W$ shown in Figure~\ref{fig:asymptotic2}).

\subsection{Experiments}

By Theorem~\ref{thm:asymptotic1} and~\ref{thm:asymptotic2}, Kemeny and $f_B^1$ are asymptotically equal when the data are generated from
$\mm_\varphi^1$ or $\mm_\varphi^2$.  Hence, we focus on the comparison
between rule $f_B^2$ and Kemeny using synthetic data generated from
$\mm_\varphi^2$ given the binary relation $W$ illustrated in
Figure~\ref{fig:asymptotic2}.

By Theorem~\ref{thm:comp}, the exact computation of Bayesian risk involves computing $\varphi^{\Omega(n)}$, which is exponentially small for large $n$ since $\varphi<1$. Hence, we need a special data structure to handle the computation of $f_B^2$, because a straightforward implementation easily loses precision. In our experiments, we use the following approximation for $f_B^2$:

\begin{dfn}\label{dfn:g} For any $c\in\mc$ and profile $P$, let $s(c,P) = \sum_{b:w_{P}(b,c)>0}w_P(b,c)$. Let $g$ be the voting rule such that for any profile $P$, $g(P)=\arg\min_cs(c,P)$.
\end{dfn}

In words, $g$ selects the alternative $c$ with the minimum total weight on the incoming edges in the WMG. By Theorem~\ref{thm:comp}, a $f_B^2$ winner $c$ maximizes $\prod_{b\neq c}\frac{\varphi^{P(b\succ c)}}{K_{\{c,b\}}}=\prod_{b\neq c}\frac{1}{1+\varphi^{w_P(c,b)}}$, which means that $c$ minimizes $\prod_{b\neq c}(1+\varphi^{w_P(c,b)})$. In our experiments, $\prod_{b\neq c}(1+\varphi^{w_P(c,b)})$ is $(1+o(1))\varphi^{\sum_{b:w_{P}(b,c)>0}w_P(b,c)}$ for reasonably large $n$. Therefore, $g$ is a good approximation of $f_B^2$ with reasonably large $n$. Formally, this is stated in the following theorem.

\begin{thm}\label{thm:app} For any $W\in\mb(\mc)$ and any
  $\varphi$, $f_B^2(P_n)=g(P_n)$ a.a.s.~as $n\ra\infty$ and
  votes in $P_n$ are generated i.i.d.~from $\mm_\varphi^2$ given $W$.\end{thm}

In our experiments, data are generated by $\mm_\varphi^2$ given $W$ in Figure~\ref{fig:asymptotic2} for $m=5$, $n\in \{100,200,\ldots,2000\}$, and $\varphi\in\{0.1,0.5,0.9\}$. For each setting we generate $1500$ profiles, and calculate the percentage for $g$ and $\kemeny$ to be different. The results are shown in Figuire~\ref{fig:exp}.
We observe that for $\varphi=0.1$ and $0.5$, the probability for $g(P_n)\neq \kemeny(P_n)$ is about $30\%$ for most $n$ in our experiments; when $\varphi=0.9$, the probability is about $10\%$. In light of Theorem~\ref{thm:app}, these results confirm Theorem~\ref{thm:asymptotic2}. We have also conducted similar experiments for $\mm_\varphi^1$, and found that the $g$ winner is the same as the $\kemeny$ winner in all $10000$ randomly generated profiles with $m=5,n=100$. This provides a sanity check for Theorem~\ref{thm:asymptotic1}. 
\begin{figure}
\centering
\includegraphics[width=.45\textwidth]{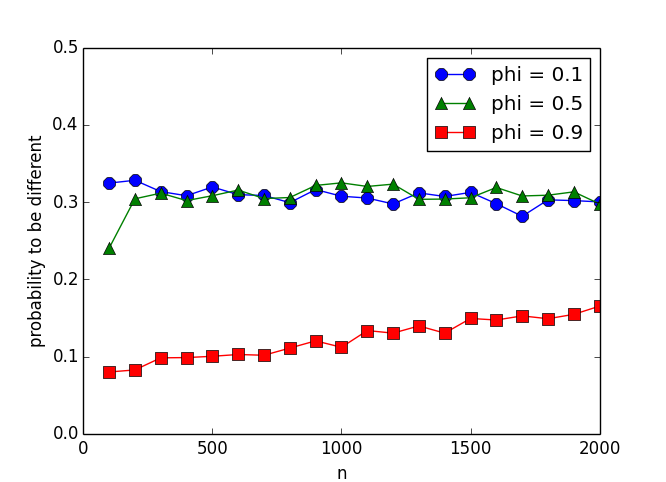}
\caption{\small Probability that $g$ is different from $\kemeny$ under $\mm_\varphi^2$.\label{fig:exp}}
\end{figure}

\section{Conclusions}

There are some immediate open questions for future work, including the
characterization of the exact computational complexity of $f_B^1$, and
the normative properties of $g$. More generally, it is interesting to study the
design and analysis of new voting rules using the proposed statistical
decision-theoretic framework under alternative probabilistic models,
e.g.~random utility models, other loss functions, e.g.~a smoother loss
function, and other sample spaces including partial orders of a fixed
set of $k$ alternatives.  We also plan to design and evaluate randomized
estimators, and estimators that minimizes the maximum expected loss or the maximum expected regret~\citep{Berger85:Statistical}.

\section{Acknowledgments} We thank Shivani Agarwal, Craig Boutilier, Yiling Chen, Vincent Conitzer, Edith Elkind, Ariel Procaccia, and anonymous reviewers of AAAI-14 and NIPS-14 for helpful suggestions and discussions. Azari Soufiani acknowledges Siebel foundation for the scholarship in his last year of PhD studies. Parkes was supported in part by NSF grant CCF \#1301976 and the SEAS TomKat fund. Xia acknowledges an RPI startup fund for support.

%
%

{


}

\end{document}